\documentclass[final,12pt]{arxiv}

\usepackage{times}
\usepackage{mathtools}
\usepackage[mathscr]{euscript}
\usepackage{thmtools}
\usepackage{thm-restate}
\usepackage{booktabs}  
\usepackage{caption}
\usepackage{graphicx}
\usepackage{wrapfig}

\usepackage{bm}

\let\vec\undefined
\usepackage{MnSymbol}
\DeclareMathAlphabet\mathbb{U}{msb}{m}{n}
\usepackage{xpatch}

\def\Rset{\mathbb{R}}

\DeclareMathOperator*{\E}{\mathbb E}

\DeclareMathOperator*{\argmin}{argmin}

\DeclarePairedDelimiter{\abs}{\lvert}{\rvert} 
\DeclarePairedDelimiter{\bracket}{[}{]}
\DeclarePairedDelimiter{\curl}{\{}{\}}
\DeclarePairedDelimiter{\norm}{\lVert}{\rVert}
\DeclarePairedDelimiter{\paren}{(}{)}

\hypersetup{
  breaklinks   = true, 
  colorlinks   = true, 
  urlcolor     = blue, 
  linkcolor    = blue, 
  citecolor    = blue 
}

\newcommand{\sA}{{\mathscr A}}

\newcommand{\sC}{{\mathscr C}}
\newcommand{\sD}{{\mathscr D}}

\newcommand{\sF}{{\mathscr F}}

\newcommand{\sH}{{\mathscr H}}

\newcommand{\sM}{{\mathscr M}}

\newcommand{\sR}{{\mathscr R}}

\newcommand{\sT}{{\mathscr T}}

\newcommand{\sX}{{\mathscr X}}
\newcommand{\sY}{{\mathscr Y}}

\newcommand{\by}{{\mathbf y}}

\newcommand{\bw}{{\mathbf w}}

\newcommand{\sfL}{{\mathsf L}}

\newcommand{\ov}{\overline}
\newcommand{\wt}{\widetilde}
\newcommand{\e}{\epsilon}

\newcommand{\ignore}[1]{}

\DeclareMathOperator{\sgn}{sgn}
\DeclareMathOperator{\sign}{sign}
\def\Nset{\mathbb{N}}
\newcommand{\hh}{{\sf h}}
\newcommand{\yy}{{\sf y}}

\newcommand{\q}{{q}}
\newcommand{\num}{l}
\newcommand{\tp}{\textsf{TP}}
\newcommand{\tn}{\textsf{TN}}
\newcommand{\fp}{\textsf{FP}}
\newcommand{\fn}{\textsf{FN}}

\newcommand{\sur}{\wt \sfL}
\newcommand{\ch}{\mathsf c_{\hh}}
\newcommand{\cy}{\mathsf c_{\yy}}
\newcommand{\cyy}{\mathsf c_{y'}}
\newcommand{\s}{{\sf s}}
\newcommand{\qq}{{\sf q}}
\newcommand{\z}{{\sf z}}
\newcommand{\sh}{\mathsf s_{\hh}}
\newcommand{\sy}{\mathsf s_{\yy}}
\newcommand{\syy}{\mathsf s_{y'}}
\newcommand{\zh}{\mathsf z_{\hh}}
\newcommand{\zy}{\mathsf z_{\yy}}
\newcommand{\zyy}{\mathsf z_{y'}}

\usepackage[toc, page, header]{appendix}
\title[Multi-Label Learning with Stronger Consistency Guarantees]{Multi-Label Learning with Stronger Consistency Guarantees}

\arxivauthor{%
 \Name{Anqi Mao} \Email{aqmao@cims.nyu.edu}\\
 \addr Courant Institute of Mathematical Sciences, New York%
 \AND
 \Name{Mehryar Mohri} \Email{mohri@google.com}\\
 \addr Google Research and Courant Institute of Mathematical Sciences, New York%
 \AND
 \Name{Yutao Zhong} \Email{yutao@cims.nyu.edu}\\
 \addr Courant Institute of Mathematical Sciences, New York%
}

\begin{document}

\maketitle

\begin{abstract}

We present a detailed study of surrogate losses and algorithms for
multi-label learning, supported by $\sH$-consistency bounds. We first
show that, for the simplest form of multi-label loss (the popular
Hamming loss), the well-known consistent binary relevance surrogate
suffers from a sub-optimal dependency on the number of labels in terms
of $\sH$-consistency bounds, when using smooth losses such as logistic
losses. Furthermore, this loss function fails to account for label
correlations.  To address these drawbacks, we introduce a novel
surrogate loss, \emph{multi-label logistic loss}, that accounts for
label correlations and benefits from label-independent
$\sH$-consistency bounds.  We then broaden our analysis to cover a
more extensive family of multi-label losses, including all common ones
and a new extension defined based on linear-fractional functions with
respect to the confusion matrix. We also extend our multi-label
logistic losses to more comprehensive multi-label comp-sum losses,
adapting comp-sum losses from standard classification to the
multi-label learning. We prove that this family of surrogate losses
benefits from $\sH$-consistency bounds, and thus Bayes-consistency,
across any general multi-label loss. Our work thus proposes a unified
surrogate loss framework benefiting from strong consistency guarantees
for any multi-label loss, significantly expanding upon previous work
which only established Bayes-consistency and for specific loss
functions. Additionally, we adapt constrained losses from standard
classification to multi-label constrained losses in a similar way,
which also benefit from $\sH$-consistency bounds and thus
Bayes-consistency for any multi-label loss. We further describe
efficient gradient computation algorithms for minimizing the
multi-label logistic loss.

\end{abstract}



\section{Introduction}
\label{sec:intro}

Supervised learning methods often assign a single label to each
instance.  However, real-world data exhibits a more complex structure,
with objects belonging to multiple categories simultaneously. Consider
a video about sports training, which could be categorized as both
`health' and `athletics,' or a culinary blog post tagged with
`cooking' and `nutrition'.  As a result, multi-label learning
\citep{mccallum1999multi,schapire2000boostexter} has become
increasingly important, leading to the development of various
interesting and effective approaches, predominantly experimental in
nature, in recent years
\citep{elisseeff2001kernel,deng2011fast,petterson2011submodular,
  kapoor2012multilabel}.

Although there is a rich literature on multi-label learning (see
\citep{zhang2013review} and \citep{bogatinovski2022comprehensive} for
detailed surveys), only a few studies focus on the theoretical
analysis of multi-label learning, particularly the study of
the Bayes-consistency of surrogate losses
\citep{Zhang2003,zhang2004statistical,bartlett2006convexity,
  tewari2007consistency,steinwart2007compare}.

\citet{gao2011consistency} initiated the study of Bayes-consistency in
multi-label learning with respect to Hamming loss and (partial)
ranking loss. They provided negative results for ranking loss,
demonstrating that no convex and differentiable pairwise surrogate
loss is Bayes-consistent for that multi-label loss.  They also showed
that the \emph{binary relevance} method, which learns an independent
binary classifier for each of the $\num$ labels, is Bayes-consistent
with respect to the Hamming loss.  \citet{dembczynski2011exact}
further demonstrated that under the assumption of conditionally
independent labels, the \emph{binary relevance} method is also
Bayes-consistent with respect to the $F_{\beta}$ measure
loss. However, they noted that it can perform arbitrarily poorly when
this assumption does not hold.  \citet{dembczynski2012consistent}
provided a positive result for the (partial) ranking loss by showing
that the simpler univariate variants of smooth surrogate losses are
Bayes-consistent with respect to it. Additionally,
\citet{zhang2020convex} proposed a family of Bayes-consistent
surrogate losses for the $F_{\beta}$ measure by reducing the
$F_{\beta}$ learning problem to a set of binary class probability
estimation problems. This approach was motivated by the consistent
output coding scheme in \citep{ramaswamy2014consistency} for general
multiclass problems. Other works have studied generalization bounds in multi-label learning \citep{yu2014large,wydmuch2018no,wu2020multi,wu2021rethinking,wu2023towards,busa2022regret}.

Another related topic is the characterization of the Bayes classifier
and corresponding Bayes-consistent plug-in algorithm in multi-label
learning. This includes the characterization of the Bayes classifier
for subset $0/1$ loss and Hamming loss in \citep{cheng2010bayes} and
the characterization of the Bayes classifier for $F_1$ measure in
\citep{dembczynski2011exact}. \citet{dembczynski2013optimizing,
  waegeman2014bayes} further extended the results in
\citep{dembczynski2011exact} by designing a Bayes-consistent plug-in
algorithm for the $F_{\beta}$ measure. \citet{koyejo2015consistent}
characterized the Bayes classifier for general linear fractional
losses with respect to the confusion matrix and designed the
corresponding plug-in algorithms in the empirical utility maximization
(EUM) framework. In this framework, the measures are directly defined
as functions of the population, in contrast to a loss function that is
defined as a function over a single instance in the decision theoretic
analysis (DTA) framework
\citep{ye2012optimizing}. \citet{menon2019multilabel} studied the
Bayes-consistency of various reduction methods with respect to
Precision@$\kappa$ and Recall@$\kappa$ in multi-label
learning. However, all these publications only established
Bayes-consistency for specific loss functions.  Can we derive a
unified surrogate loss framework that is Bayes-consistent for any
multi-label loss?

Furthermore, as \citet{awasthi2022Hconsistency,awasthi2022multi}
pointed out, Bayes-consistency is an asymptotic guarantee and does not
provide convergence guarantees. It also applies only to the family of
all measurable functions unlike the restricted hypothesis sets
typically used in practice.  Instead, they proposed a stronger
guarantee known as \emph{$\sH$-consistency bounds}, which are both
non-asymptotic and account for the hypothesis set while implying
Bayes-consistency. These guarantees provide upper bounds on the target
estimation error in terms of the surrogate estimation error.  Can we
leverage this state-of-the-art consistency guarantee when designing
surrogate loss functions for multi-label learning?

Moreover, one of the main concerns in multi-label learning is label
correlations (see \citep{dembczynski2012label}). For the simplest form
of multi-label loss, the popular Hamming loss, the existing
Bayes-consistent binary relevance surrogate fails to account for label
correlations. Can we design consistent loss functions that effectively
account for label correlations as well?

\textbf{Our Contributions.} This paper directly addresses these key
questions in multi-label learning.  We present a detailed study of
surrogate losses and algorithms for multi-label learning, supported by
$\sH$-consistency bounds.

In Section~\ref{sec:ham}, we first show that for the simplest form of
multi-label loss, the popular Hamming loss, the well-known consistent
binary relevance surrogate, when using smooth losses such as logistic
losses, suffers from a sub-optimal dependency on the number of labels
in terms of $\sH$-consistency bounds. Furthermore, this loss function
fails to account for label correlations.

To address these drawbacks, we introduce a novel surrogate loss,
\emph{multi-label logistic loss}, that accounts for label correlations
and benefits from label-independent $\sH$-consistency bounds
(Section~\ref{sec:log}).  We then broaden our analysis to cover a more
extensive family of multi-label losses, including all common ones and
a new extension defined based on linear-fractional functions with
respect to the confusion matrix (Section~\ref{sec:general}).

In Section~\ref{sec:comp}, we also extend our multi-label logistic
losses to more comprehensive multi-label comp-sum losses, adapting
comp-sum losses from standard classification to the multi-label
learning. We prove that this family of surrogate losses benefits from
$\sH$-consistency bounds, and thus Bayes-consistency, across any
general multi-label loss. Our work thus proposes a unified surrogate
loss framework that is Bayes-consistent for any multi-label loss,
significantly expanding upon previous work which only established
consistency for specific loss functions.

Additionally, we adapt constrained losses from standard classification
to multi-label constrained losses in a similar way, which also benefit
from $\sH$-consistency bounds and thus Bayes-consistency for any
multi-label loss (Appendix~\ref{sec:cstnd}). We further describe
efficient gradient computation algorithms for minimizing the
multi-label logistic loss (Section~\ref{sec:algo}).

\section{Preliminaries}
\label{sec:pre}

\textbf{Multi-label learning.} We consider the standard multi-label
learning setting. Let $\sX$ be the input space and $\sY = \curl*{+1,
  -1}^\num$ the set of all possible labels or tags, where $\num$ is a
finite number. For example, $\sX$ can be a set of images, and $\sY$
can be a set of $\num$ pre-given tags (such as 'flowers', 'shoes', or
'books') that can be associated with each image in the image tagging
problem. Let $n = \abs*{\sY}$. For any instance $x \in \sX$ and its
associated label $y = \paren*{y_1, \ldots, y_{\num}} \in \sY$, if $y_i
= +1$, we say that label $i$ is relevant to $x$. Otherwise, it is not
relevant. Let $[\num] = \curl*{1, \ldots, \num}$. Given a sample $S$
drawn i.i.d.\ according to some distribution $\sD$ over $\sX \times
\sY$, the goal of multi-label learning is to learn a hypothesis $h
\colon \sX \times [\num] \to \Rset$ to minimize the generalization
error defined by a multi-label loss function $\sfL \colon
\sH_{\rm{all}} \times \sX \times \sY \to \Rset$,
\begin{equation*}
    \sR_{\sfL}(h) = \E_{(x, y) \sim \sD} \bracket*{ \sfL(h, x, y) },
\end{equation*}
where $\sH_{\rm{all}}$ is the family of all measurable hypotheses. For
convenience, we abusively denote the scoring vector by $h(x) =
\paren*{h(x, 1), \ldots, h(x, \num)}$. Given a hypothesis set $\sH
\subset \sH_{\rm{all}}$, we denote by $\sR^*_{\sfL}(\sH) = \inf_{h \in
  \sH} \sR(h)$ the best-in-class error. We refer to the difference
$\sR_{\sfL}(h) - \sR^*_{\sfL}(\sH)$ as the \emph{estimation error},
which is termed the \emph{excess error} when $\sH =
\sH_{\rm{all}}$. Let $\sgn \colon t \mapsto 1_{t \geq 0} - 1_{t < 0}$
be the sign function, and let $t \colon \sX \to \Rset_{+}$ be a
threshold function. The target loss function $\sfL$ can be typically
given by a function $\ov \sfL$ mapping from $\sY \times \sY$ to real
numbers:
\begin{equation}
\label{eq:target}
\sfL(h, x, y) = \ov \sfL \paren*{\hh(x), y},
\end{equation}
where $\hh(x) \coloneqq \bracket*{\hh_1(x), \ldots, \hh_\num(x)} \in
\sY$ is the prediction for the input $x \in \sX$ and $\hh_i(x) =
\sign(h(x, i) - t(x))$ for any $i \in [\num]$. As with many
multi-label learning algorithms, such as binary relevance, we set the
threshold function $t(x) = 0$.  There are many multi-label loss
functions, such as Hamming loss, (partial) ranking loss, $F_1$ and the
more general $F_{\beta}$ measure loss, subset $0/1$ loss,
precision@$\kappa$, recall@$\kappa$,
etc. \citep{zhang2013review}. Among these, several loss functions are
defined based on the prediction of the hypothesis $\hh(x)$, while
others are based on the scoring vector $h(x)$. We will specifically
consider the first type of multi-label loss in the form given in
\eqref{eq:target}, which is based on some `distance' between the
prediction and the true label. This includes all the loss functions
previously mentioned (see Section~\ref{sec:general} for a list of
several common multi-label losses in this family) but excludes the
(partial) ranking loss, which is defined based on pairwise scores. For
convenience, we may alternatively refer to $\ov \sfL$ or its induced
$\sfL$ as the multi-label loss. Without loss of generality, we assume
that $\ov \sfL \in [0, 1]$, which can be achieved through
normalization. We also denote by $\sfL_{\max} = \max_{y', y} \ov
\sfL(y', y)$. Our analysis is general and adapts to any multi-label
loss $\ov \sfL$.

\textbf{Surrogate risk minimization and consistency.} Minimizing the
multi-label loss $\sfL$ directly is computationally hard for most
hypothesis sets because it is discrete and non-convex. A common method
involves minimizing a smooth surrogate loss function $\sur \colon
\sH_{\rm{all}} \times \sX \times \sY \to \Rset$, which is the main
focus of this paper. Minimizing a surrogate loss directly leads to an
algorithm for multi-label learning. A desirable guarantee for the
surrogate loss in multi-label learning is \emph{Bayes-consistency}
\citep{Zhang2003,zhang2004statistical,bartlett2006convexity,tewari2007consistency,steinwart2007compare,gao2011consistency}. That
is, minimizing the surrogate loss over the family of all measurable
functions leads to the minimization of the multi-label loss over the
same family:
\begin{definition}
A surrogate loss $\sur$ is said to be \emph{Bayes-consistent} with respect to a multi-label loss $\sfL$ if the following holds for any distribution and all given sequences of hypotheses $\curl*{h_n}_{n \in \Nset} \subset \sH_{\rm{all}}$:
\begin{equation*}
\paren*{ \lim_{n \to \plus \infty} \sR_{\sur}(h_n) - \sR^*_{\sur}(\sH_{\rm{all}}) = 0 } \implies \paren*{ \lim_{n \to \plus \infty} \sR_{\sfL}(h_n) - \sR^*_{\sfL}(\sH_{\rm{all}}) = 0}.
\end{equation*}
\end{definition}
As pointed out by \citet{awasthi2022Hconsistency,awasthi2022multi} (see also
\citep{long2013consistency,zhang2020bayes,MaoMohriZhong2023rankingabs,MaoMohriZhong2023ranking,MaoMohriMohriZhong2023twostage,MaoMohriZhong2023characterization,MaoMohriZhong2023structured,AwasthiMaoMohriZhong2023theoretically,awasthi2023dc,MaoMohriZhong2024deferral,MaoMohriZhong2024predictor,MaoMohriZhong2024score,mao2024top,mao2024universal,mao2024regression,mao2024h,MohriAndorChoiCollinsMaoZhong2023learning,cortes2024cardinality}),
Bayes-consistency is an asymptotic guarantee that cannot provide any
guarantee for approximate minimizers; it also applies only to the
family of all measurable functions and does not consider the
hypothesis sets typically used in practice. Instead, they propose a
stronger guarantee known as \emph{$\sH$-consistency bounds}, which are
both non-asymptotic and dependent on the hypothesis set, and imply
Bayes-consistency when $\sH = \sH_{\mathrm{all}}$. These guarantees
provide upper bounds on the target estimation error in terms of the
surrogate estimation error. In the multi-label learning scenario, they
can be formulated as follows:
\begin{definition}
A surrogate loss $\sur$ is said to admit an \emph{$\sH$-consistency
bound} with respect to a multi-label loss $\sfL$ if the following
condition holds for any distribution and for all hypotheses $h \in
\sH$, given a concave function $\Gamma \colon \Rset_{+} \to \Rset_{+}$
with $\Gamma(0) = 0$:
\begin{equation}
\label{eq:H-consistency-bound}
\sR_{\sfL}(h) - \sR^*_{\sfL}(\sH) + \sM_{\sfL}(\sH) \leq \Gamma \paren*{ \sR_{\sur}(h) - \sR^*_{\sur}(\sH) + \sM_{\sur}(\sH) }.
\end{equation}
\end{definition}
The quantities $\sM_{\sur}(\sH)$ appearing in the bounds are called
minimizability gaps, which measure the difference between the
best-in-class error and the expected best pointwise error for a loss
function $\sur$ and a hypothesis set $\sH$:
\begin{equation*}
\sM_{\sur}(\sH) = \sR^*_{\sur}(\sH) - \E_{x} \bracket*{ \inf_{h \in \sH} \paren*{\E_{y \mid x} \bracket*{\sur(h, x, y)}}} \geq 0.
\end{equation*}
These are inherent quantities depending on the distribution and
hypothesis set, which we cannot hope to minimize. Since $\Gamma$ is
concave and $\Gamma(0) = 0$, $\Gamma$ is sub-additive and an
$\sH$-consistency bound \eqref{eq:H-consistency-bound} implies that:
$\sR_{\sfL}(h) - \sR^*_{\sfL}(\sH) + \sM_{\sfL}(\sH) \leq \Gamma
\paren*{ \sR_{\sur}(h) - \sR^*_{\sur}(\sH)} + \Gamma
\paren*{\sM_{\sur}(\sH) }$. Therefore, when the surrogate estimation
error $\paren*{\sR_{\sur}(h) - \sR^*_{\sur}(\sH)}$ is minimized to
$\epsilon$, the target estimation error $\paren*{\sR_{\sfL}(h) -
  \sR^*_{\sfL}(\sH)}$ is upper bounded by $\Gamma(\e) + \Gamma
\paren*{\sM_{\sur}(\sH)}$. The minimizability gaps vanish when $\sH =
\sH_{\rm{all}}$ or in more general realizable cases, such as when
$\sR^*_{\sur}(\sH) = \sR^*_{\sur}(\sH_{\rm{all}})$
\citep{steinwart2007compare,awasthi2022multi,mao2023cross}. In these
cases, $\sH$-consistency bounds imply the $\sH$-consistency of a
surrogate loss $\sur$ with respect to a multi-label loss $\sfL$:
$\sR_{\sur}(h) - \sR^*_{\sur}(\sH) \leq \e \implies \sR_{\sfL}(h) -
\sR^*_{\sfL}(\sH) \leq \Gamma(\e)$, for any $\e \geq 0$. The
minimizability gap $\sM_{\sur}(\sH)$ is upper bounded by the
approximate error $\sA_{\sur}(\sH) = \sR^*_{\sur}(\sH) - \E_{x}
\bracket*{ \inf_{h \in \sH_{\rm{all}}} \paren*{\E_{y \mid x}
    \bracket*{\sur(h, x, ,y)}}}$ and is generally a finer quantity
\citep{mao2023cross}. Thus, $\sH$-consistency bounds are more
informative, more favorable, and stronger than excess error bounds,
and they imply these bounds when $\sH = \sH_{\mathrm{all}}$.

Next, we will study surrogate loss functions and algorithms for
multi-label learning, supported by $\sH$-consistency bounds, the
state-of-the-art consistency guarantee for surrogate risk
minimization.

\section{Existing consistent surrogates for the Hamming loss}
\label{sec:ham}

In the section, we consider the simplest form of multi-label loss, the Hamming loss, defined as:
\begin{equation*}
\forall (h, x, y) \in \sH \times \sX \times \sY, \quad 
\sfL_{\rm{ham}}(h, x, y) = \ov \sfL_{\rm{ham}}(\hh(x), y), \text{ where } \ov \sfL_{\rm{ham}}(y', y) = \sum_{i = 1}^\num 1_{y_i \neq y'_i}.
\end{equation*}
The existing Bayes-consistent surrogate loss function is to transform
the multi-label learning into $\num$ independent binary classification
tasks \citep{gao2011consistency}, defined as for all $(h, x, y) \in
\sH \times \sX \times \sY$,
\begin{equation*}
\sur_{\rm{br}}(h, x, y) = \sum_{i = 1}^\num \Phi(y_i h(x, i)),
\end{equation*}
where $\Phi \colon \Rset \to \Rset_{+} $ is a binary margin-based loss
function, such as the logistic loss $u \mapsto \log \paren*{1 +
  e^{-u}}$. The algorithm that minimizes this surrogate loss is known
as \emph{binary relevance} \citep{zhang2013review}, which learns an
independent binary classifier for each of the $\num$ labels.
\citet[Theorem~15]{gao2011consistency} shows that $\sur_{\rm{br}}$ is
Bayes-consistent with respect to $\sfL_{\rm{ham}}$ if $\Phi$ is
Bayes-consistent with respect to $\ell_{0-1} \colon (f, x, y) \mapsto
1_{y \neq \sign(f(x))}$, the binary zero-one loss. Here, we prove a
stronger result that $\sur_{\rm{br}}$ admits an $\sH$-consistency
bound with respect to $\sfL_{\rm{ham}}$ with a functional form $\num
\Gamma\paren*{\frac{\cdot}{\num}}$ if $\Phi$ admits an
$\sH$-consistency bounds with respect to $\ell_{0-1}$ with a
functional form $\Gamma(\cdot)$. Let $\sF$ be a hypothesis set consist
of functions mapping from $\sX$ to $\Rset$.

\begin{restatable}{theorem}{BoundBi}
\label{Thm:binary-bound}
Let $\sH = \sF^\num$. Assume that the following $\sF$-consistency
bound holds in the binary classification, for some concave function
$\Gamma \colon \Rset \to \Rset_{+}$:
\begin{equation*}
\forall f \in \sF, \quad \sR_{\ell_{0-1}}(f) - \sR^*_{\ell_{0-1}}(\sF) + \sM_{\ell_{0-1}}(\sF) \leq \Gamma(\sR_{\Phi}(f) - \sR^*_{\Phi}(\sF) + \sM_{\Phi}(\sF) ).
\end{equation*}
Then, the following $\sH$-consistency bound holds in the multi-label learning: for all $h \in \sH$,
\begin{align*}
\sR_{\sfL_{\rm{ham}}}(h) - \sR^*_{\sfL_{\rm{ham}}}(\sH) + \sM_{\sfL_{\rm{ham}}}(\sH) 
& \leq \num \Gamma\paren*{ \frac{\sR_{\sur_{\rm{br}}}(h) - \sR^*_{\sur_{\rm{br}}}(\sH) + \sM_{\sur_{\rm{br}}}(\sH)}{\num} }.
\end{align*}
\end{restatable}
The proof is included in Appendix~\ref{app:binary-bound}. We say that
a hypothesis set $\sF$ is complete if $\curl*{f(x) \colon f \in \sF} =
\Rset$, $\forall x \in \sX$. This notion of completeness is broadly
applicable and holds for commonly used hypothesis sets in practice,
including linear hypotheses, multi-layer feed-forward neural networks,
and all measurable functions. For such complete hypothesis sets $\sF$
and with smooth functions $\Phi$ like the logistic loss function,
$\Gamma$ admits a square root dependency in the binary classification
\citep{awasthi2022Hconsistency,mao2024universal}. Thus, by
Theorem~\ref{Thm:binary-bound}, we obtain the following result.

\begin{restatable}{corollary}{BoundBiLog}
\label{cor:binary-bound-log}
Let $\sH = \sF^\num$. Assume that $\sF$ is complete and $\Phi(u) = \log(1 + e^{-u})$.
Then, the following $\sH$-consistency bound holds in the multi-label learning: for all $h \in \sH$,
\begin{equation*}
 \sR_{\sfL_{\rm{ham}}}(h) - \sR^*_{\sfL_{\rm{ham}}}(\sH) + \sM_{\sfL_{\rm{ham}}}(\sH) 
\leq \num^{\frac12} \paren*{\sR_{\sur_{\rm{br}}}(h) - \sR^*_{\sur_{\rm{br}}}(\sH) + \sM_{\sur_{\rm{br}}}(\sH) }^{\frac12}.
\end{equation*}
\end{restatable}
Since $t \mapsto t^{\frac12}$ is sub-additive, the right-hand side of
the $\sH$-consistency bound in Corollary~\ref{cor:binary-bound-log}
can be further upper bounded by $\num^{\frac12} \paren*{
  \sR_{\sur_{\rm{br}}}(h) - \sR^*_{\sur_{\rm{br}}}(\sH)}^{\frac12} +
\num^{\frac12} \paren*{\sM_{\sur_{\rm{br}}}(\sH) }^{\frac12}$. This
implies that when the estimation error of the surrogate loss
$\sur_{\rm{br}}$ is reduced to $\e$, the corresponding estimation
error of the Hamming loss is upper bounded by $\num^{\frac12}
\e^{\frac12} + \num^{\frac12} \paren*{\sM_{\sur_{\rm{br}}}(\sH)
}^{\frac12} - \sM_{\sfL_{\rm{ham}}}(\sH) $. In the nearly realizable
cases where minimizability gaps are negligible, this upper bound
approximates to
\begin{equation}
\label{eq::approx-br}
  \sR_{\sfL_{\rm{ham}}}(h) - \sR^*_{\sfL_{\rm{ham}}}(\sH) \leq \num^{\frac12} \e^{\frac12}. 
\end{equation}
Therefore, as the number of labels $\num$ increases, the bound becomes less favorable. Furthermore, the loss function $\sur_{\rm{br}}$ clearly fails to account for the inherent correlations among labels. For instance, `coffee' and 'mug' are more likely to co-occur than `coffee' and `umbrella'. Additionally, $\sur_{\mathrm{br}}$ is only Bayes-consistent with respect to the Hamming loss and cannot yield risk-minimizing predictions for other multi-label losses such as subset $0/1$ loss or $F_{\beta}$-measure loss \citep{dembczynski2012label}. To address these drawbacks, we will introduce a new surrogate loss in the next section.

\section{Multi-label logistic loss}
\label{sec:log}

In this section, we define a new surrogate loss for Hamming loss in
multi-label learning that accounts for label correlations and benefits
from label-independent $\sH$-consistency bounds. This loss function
can be viewed as a generalization of the (multinomial) logistic loss
\citep{Verhulst1838,Verhulst1845,Berkson1944,Berkson1951}, used in
standard classification, to multi-label learning.  Thus, we will refer
to it as \emph{multi-label logistic loss}. It is defined as follows:
for all $(h, x, y) \in \sH \times \sX \times \sY$,
\begin{equation*}
\sur_{\log}(h, x, y)  = \sum_{y' \in \sY} \paren*{ 1- \ov \sfL_{\rm{ham}}(y', y) } \log \paren*{\sum_{y'' \in \sY} e^{\sum_{i = 1}^\num \paren*{ y''_i - y'_i} h(x, i)}}.
\end{equation*}
This formulation can be interpreted as a weighted logistic loss, where
$\paren*{ 1- \ov \sfL_{\rm{ham}}(\cdot, y) } $ serves as a weight
vector. Additionally, this formulation accounts for label correlations
among the $y_i$s within the logarithmic function.

The next result shows that the multi-label logistic loss benefits from
a favorable $\sH$-consistency bound with respect to $\sfL_{\rm{ham}}$,
without dependency on the number of labels $\num$. We assume that $\sH
= \sF^{\num}$ and $\sF$ is complete, conditions that typically hold in
practice.
\begin{restatable}{theorem}{BoundLog}
\label{Thm:log-bound}
Let $\sH = \sF^\num$. Assume that $\sF$ is complete. 
Then, the following $\sH$-consistency bound holds in the multi-label learning: for all $h \in \sH$,
\begin{equation*}
\sR_{\sfL_{\rm{ham}}}(h) - \sR^*_{\sfL_{\rm{ham}}}(\sH) + \sM_{\sfL_{\rm{ham}}}(\sH) 
\leq 2 \paren*{ \sR_{\sur_{\log}}(h) - \sR^*_{\sur_{\log}}(\sH) + \sM_{\sur_{\log}}(\sH) }^{\frac12}.
\end{equation*}
\end{restatable}
Since $t \mapsto t^{\frac12}$ is sub-additive, the right-hand side of
the $\sH$-consistency bound in Theorem~\ref{Thm:log-bound} can be
further upper bounded by $ 2 \paren*{ \sR_{\sur_{\log}}(h) -
  \sR^*_{\sur_{\log}}(\sH)}^{\frac12} + 2
\paren*{\sM_{\sur_{\log}}(\sH) }^{\frac12}$. This implies that when
the estimation error of the surrogate loss $\sur_{\log}$ is reduced up
to $\e$, the corresponding estimation error of the Hamming loss is
upper bounded by $2 \e^{\frac12} + 2 \paren*{\sM_{\sur_{\log}}(\sH)
}^{\frac12} - \sM_{\sfL_{\rm{ham}}}(\sH) $. In the nearly realizable
cases where minimizability gaps are negligible, this upper bound
approximates to
\begin{equation}
\label{eq::approx-log}
  \sR_{\sfL_{\rm{ham}}}(h) - \sR^*_{\sfL_{\rm{ham}}}(\sH) \leq 2 \e^{\frac12}. 
\end{equation}
Therefore, the bound is independent of the number of labels
$\num$. This contrasts with the bound for $\sur_{\rm{br}}$ shown in
\eqref{eq::approx-br}, where a label-dependent factor
$\ell^{\frac{1}{2}}$ replaces the constant factor $2 $, making it
significantly less favorable.

The proof of Theorem~\ref{Thm:log-bound} is included in
Appendix~\ref{app:log-bound}. We first present a general tool
(Theorem~\ref{Thm:tool-Gamma}) in Appendix~\ref{app:tool}, which shows
that to derive $\sH$-consistency bounds in multi-label learning with a
concave function $\Gamma$, it is only necessary to upper bound the
conditional regret of the target multi-label loss by that of the
surrogate loss with the same $\Gamma$. This generalizes
\citet[Theorem~2]{awasthi2022multi} in standard multi-class
classification to multi-label learning. Next, we characterize the
conditional regret of the target multi-label loss, such as Hamming
loss, in Lemma~\ref{lemma:delta_target} found in
Appendix~\ref{app:tool}, under the given assumption. By using
Lemma~\ref{lemma:delta_target}, we upper bound the conditional regret
of $\sfL_{\rm{ham}}$ by that of the surrogate loss $\sur_{\log}$ with
a concave function $\Gamma(t) = 2 \sqrt{t}$.

When $\sH = \sH_{\rm{all}}$, minimizability gaps
$\sM_{\sur_{\log}}(\sH) $ and $\sM_{\sfL_{\rm{ham}}}(\sH)$ vanish,
Theorem~\ref{Thm:log-bound} implies excess error bound and
Bayes-consistency of multi-label logistic loss with respect to the
Hamming loss.
\begin{corollary}
The following excess error bound holds in the multi-label learning: for all $h \in \sH_{\rm{all}}$,
\begin{equation*}
\sR_{\sfL_{\rm{ham}}}(h) - \sR^*_{\sfL_{\rm{ham}}}(\sH_{\rm{all}})
\leq 2 \paren*{ \sR_{\sur_{\log}}(h) - \sR^*_{\sur_{\log}}(\sH_{\rm{all}})}^{\frac12}.
\end{equation*}
Moreover, $\sur_{\log}$ is Bayes-consistent with respect to $\sfL_{\rm{ham}}$. 
\end{corollary}

It is known that $\sur_{\rm{br}}$ is only Bayes-consistent with
respect to the Hamming loss and can be arbitrarily bad for other
multi-label losses such as $F_{\beta}$-measure loss
\citep{dembczynski2011exact}. Instead, we will show in the following
section that our surrogate loss $\sur_{\log}$ adapts to and is
Bayes-consistent with respect to an extensive family of multi-label
losses, including the $F_{\beta}$ measure loss.

\section{Extension: general multi-label losses}
\label{sec:general}

In this section, we broaden our analysis to cover a more extensive family of multi-label losses, including all common ones and a new extension defined based on linear-fractional functions with respect to the confusion matrix. Note that several loss functions are defined over the space $\curl*{0, 1}^{\num}$, rather than $\curl*{+1, -1}^{\num}$. To accommodate this difference, any pair $y, y' \in \sY = \curl*{+1, -1}^{\num}$ can be projected onto $\curl*{0, 1}^{\num}$ by letting $\ov y = \frac{y + \bf{1}}{2}$ and $\ov y' = \frac{y' + \bf{1}}{2}$, where ${\bf{1}} \in \Rset^{\num}$ is the vector with all elements equal to 1. Several common multi-label losses are defined as follows.

\textbf{Hamming loss:} $\ov \sfL(y', y) = \sum_{i = 1}^\num 1_{y_i \neq y'_i}$.

\textbf{$F_{\beta}$-measure loss:} $\ov \sfL(y', y) = 1 - \frac{(1 + \beta^2) \ov y' \cdot \ov y}{\beta^2 \norm*{\ov y}_1 + \norm*{\ov y'}_1}$.

\textbf{Subset $0/1$ loss:} $\ov \sfL(y', y) = \max_{i \in [\num]} 1_{y'_i \neq y_i}$.

\textbf{Jaccard distance:} $\ov \sfL(y', y) = 1 - \frac{\ov y' \cdot \ov y}{\norm*{\ov y}_1 + \norm*{\ov y'}_1 - \ov y' \cdot y}$

\textbf{Precision@$\kappa$:} $\ov \sfL(y', y) = 1 - \frac{1}{\kappa} \sum_{i \in \sT(\ov y')} 1_{\ov y_i = 1}$ subject to $y' \in \sY_{\kappa}$, where $\sY_{\kappa} = \curl*{y \in \sY \colon \norm*{\ov y}_1 = \kappa}$ and $\sT(\ov y') = \curl*{i \in [\num] \colon \ov y'_i = 1 }$.

\textbf{Recall@$\kappa$:} $\ov \sfL(y', y) = 1 - \frac{1}{\norm*{\ov y}_1} \sum_{i \in \sT(\ov y')} 1_{\ov y_i = 1}$ subject to $y' \in \sY_{\kappa}$, where $\sY_{\kappa} = \curl*{y \in \sY \colon \norm*{\ov y}_1 = \kappa}$ and $\sT(\ov y') = \curl*{i \in [\num] \colon \ov y'_i = 1 }$.

More generally, we can define a multi-label loss based on true positives (\tp),
true negatives (\tn), false positives (\fp) and false negatives (\fn) , which can be written explicitly as follows:
\begin{align*}
\tp &= \ov y' \cdot \ov y \quad \tn = \norm*{\ov y}_{1} - \ov y'\cdot \ov y\\
\fp &= \norm*{\ov y'}_{1} - \ov y' \cdot \ov y, \quad \fn = \num +  \ov y' \cdot \ov y - \norm*{\ov y}_{1} - \norm*{\ov y'}_{1}
\end{align*}
Similar to \citep{koyejo2014consistent, koyejo2015consistent}, we now define a general family of multi-label losses as linear-fractional functions in terms of these four quantities:
\begin{equation}
\label{eq:linear-frac}
\ov \sfL(y', y) = \frac{a_0 + a_{11} \tp + a_{10} \fp + a_{01} \fn + a_{00} \tn}{b_0 + b_{11} \tp + b_{10} \fp + b_{01} \fn + b_{00} \tn}.
\end{equation}
It can be shown that the aforementioned Hamming loss, $F_{\beta}$-measure loss, Jaccard distance, precision and recall all belong to this family. Note that the previous definitions in \citep{koyejo2014consistent, koyejo2015consistent} are within the empirical utility maximization (EUM) framework \citep{ye2012optimizing}, where the measures are directly defined as functions of the population. We generalize their definition to the decision theoretic analysis (DTA) framework, in terms of loss functions defined over $y$ and $y'$.

Moreover, we can consider extending multi-label losses \eqref{eq:linear-frac} to non-linear fractional functions of these four quantities, or more generally, to any other forms, as long as they are defined over the space $\sY \times \sY$.

Another important family of multi-label losses is the \emph{tree distance} loss, used in cases of hierarchical classes. In many practical applications, the class labels exist within a predefined hierarchy. For example, in the image tagging problem, class labels might include broad categories such as `animals' or `vehicles', which further subdivide into more specific classes like `mammals' and `birds' for animals, or `cars' and `trucks' for vehicles. Each of these subcategories can be divided even further, showcasing a clear hierarchical structure.

\textbf{Tree distance:} Let $T = \paren*{\sY, E, W}$
be a tree over the label space $\sY$, with edge set $E$ and positive, finite
edge lengths specified by $W$. Suppose $r \in \sY$ is designated as the root node. Then, $\ov \sfL_{T}(y', y) = \text{the shortest path length in } T \text{ between } y \text{ and } y'$.

Despite the widespread use of hierarchical classes in practice, to our knowledge, no Bayes-consistent surrogate has been proposed for the tree distance loss in multi-label learning. Next, we will show that our multi-label logistic loss can accommodate all these different loss functions, including the tree distance loss. For any general multi-label loss $\ov \sfL$, we define the multi-label logistic loss as follows:
\begin{equation}
\label{eq:log}
\forall (h, x, y) \in \sH \times \sX \times \sY, \quad 
\sur_{\log}(h, x, y) = \sum_{y' \in \sY} \paren*{ 1- \ov \sfL(y', y) } \log \paren*{\sum_{y'' \in \sY} e^{\sum_{i = 1}^\num \paren*{ y''_i - y'_i} h(x, i)}}.
\end{equation}
Here, $\ov \sfL$ can be chosen as all the multi-label losses mentioned above. Next, we will show that $\sur_{\log}$ benefits from $\sH$-consistency bounds and Bayes consistency with respect to any of these loss functions.

\begin{restatable}{theorem}{BoundLogGeneral}
\label{Thm:log-bound-general}
Let $\sH = \sF^\num$. Assume that $\sF$ is complete. 
Then, the following $\sH$-consistency bound holds in the multi-label learning:
\begin{equation*}
\forall h \in \sH, \quad \sR_{\sfL}(h) - \sR^*_{\sfL}(\sH) + \sM_{\sfL}(\sH) 
\leq 2 \paren*{ \sR_{\sur_{\log}}(h) - \sR^*_{\sur_{\log}}(\sH) + \sM_{\sur_{\log}}(\sH) }^{\frac12}.
\end{equation*}
\end{restatable}
\vskip -0.05in
The proof of Theorem~\ref{Thm:log-bound-general} is basically the same as that of Theorem~\ref{Thm:log-bound}, modulo replacing the Hamming loss $\sfL_{\rm{ham}}$ with a general multi-label loss $\sfL$. We include it in Appendix~\ref{app:log-bound-general} for completeness.
When $\sH = \sH_{\rm{all}}$, minimizability gaps $\sM_{\sur_{\log}}(\sH) $ and $\sM_{\sfL}(\sH)$ vanish, Theorem~\ref{Thm:log-bound} implies excess error bound and Bayes-consistency of multi-label logistic loss with respect to any multi-label loss. 
\begin{corollary}
\label{cor:log-bound-general}
The following excess error bound holds in the multi-label learning: for all $h \in \sH_{\rm{all}}$,
\begin{equation*}
\sR_{\sfL}(h) - \sR^*_{\sfL}(\sH_{\rm{all}})
\leq 2 \paren*{ \sR_{\sur_{\log}}(h) - \sR^*_{\sur_{\log}}(\sH_{\rm{all}})}^{\frac12}.
\end{equation*}
Moreover, $\sur_{\log}$ is Bayes-consistent with respect to $\sfL$. 
\end{corollary}
\vskip -0.05in
Corollary~\ref{cor:log-bound-general} is remarkable, as it demonstrates that a unified surrogate loss, $\sur_{\log}$, is Bayes-consistent for any multi-label loss,
significantly expanding upon previous work which only established
consistency for specific loss functions. Furthermore, Theorem~\ref{Thm:log-bound-general} provides a stronger guarantee than Bayes-consistency,  which is both non-asymptotic and specific to the hypothesis set used. 

Minimizing the multi-label logistic loss directly leads to the effective algorithm in multi-label learning. We further discuss the efficiency and practicality of this algorithm in Section~\ref{sec:algo}, where we describe efficient gradient computation.

\section{Extension: multi-label comp-sum losses}
\label{sec:comp}
\vskip -0.01in

In this section, we
further extend our multi-label logistic losses to more comprehensive
\emph{multi-label comp-sum losses}, adapting \emph{comp-sum losses} \citep{mao2023cross} from standard
classification to the multi-label learning. As shown by \citet{mao2023cross}, comp-sum loss is defined via a composition of the function $\Psi$ and a sum, and includes the logistic loss ($\Psi(u) = \log(u)$) \citep{Verhulst1838,Verhulst1845,Berkson1944,Berkson1951}, the \emph{sum-exponential loss} ($\Psi(u) = u - 1$) \citep{weston1998multi, awasthi2022multi}, the \emph{generalized cross-entropy loss} ($\Psi(u) = \frac{1}{\q}\paren*{1
  - \frac{1}{u^{\q}}}, \q \in (0,1)$) \citep{zhang2018generalized}, and the \emph{mean absolute error loss} ($\Psi(u) = 1 -
\frac{1}{u}$) \citep{ghosh2017robust} as special cases.

Given any multi-label loss $\sfL$, we will define our novel \emph{multi-label comp-sum losses} as follows:
\begin{equation}
\label{eq:comp}
\forall (h, x, y) \in \sH \times \sX \times \sY, \quad \sur_{\rm{comp}}(h, x, y) = \sum_{y' \in \sY} \paren*{ 1- \ov \sfL(y', y) } \Psi \paren*{\sum_{y'' \in \sY} e^{\sum_{i = 1}^\num \paren*{ y''_i - y'_i} h(x, i)}}.
\end{equation}
This formulation can be interpreted as a weighted comp-sum loss, where $\paren*{ 1- \ov \sfL_{\rm{ham}}(\cdot, y) } $ serves as a weight vector. Additionally, this formulation accounts for label correlations among the $y_i$s within the function $\Psi$. Next, we prove that this family of surrogate losses benefits from $\sH$-consistency bounds, and thus Bayes-consistency, across any
general multi-label loss.
\vskip -0.02in
\begin{restatable}{theorem}{BoundComp}
\label{Thm:comp-bound}
Let $\sH = \sF^\num$. Assume that $\sF$ is complete. 
Then, the following $\sH$-consistency bound holds in the multi-label learning:
\begin{equation*}
\forall h \in \sH, \quad \sR_{\sfL}(h) - \sR^*_{\sfL}(\sH) + \sM_{\sfL}(\sH) \leq \Gamma \paren*{ \sR_{\sur_{\rm{comp}}}(h) - \sR^*_{\sur_{\rm{comp}}}(\sH) + \sM_{\sur_{\rm{comp}}}(\sH) },    
\end{equation*}
where $\Gamma(t) = 2\sqrt{t}$ when
$\Psi(u) = \log(u)$ or
$u - 1$; $\Gamma(t) = 2\sqrt{n^{\q}t}$ when
$\Psi(u) = \frac{1}{\q}\paren*{1
  - \frac{1}{u^{\q}}}, \q \in (0,1)$; and
$\Gamma(t) = n t$ when
$\Psi(u) = 1 -
\frac{1}{u}$.
\end{restatable}
\vskip -0.02in
\begin{restatable}{corollary}{BoundCompAll}
\label{Thm:comp-bound-all}
The following excess error bound holds in the multi-label learning:
\begin{equation*}
\forall h \in \sH_{\rm{all}}, \quad \sR_{\sfL}(h) - \sR^*_{\sfL}(\sH_{\rm{all}}) \leq \Gamma \paren*{ \sR_{\sur_{\rm{comp}}}(h) - \sR^*_{\sur_{\rm{comp}}}(\sH_{\rm{all}})},   
\end{equation*}
where $\Gamma(t) = 2\sqrt{t}$ when
$\Psi(u) = \log(u)$ or
$u - 1$; $\Gamma(t) = 2\sqrt{n^{\q}t}$ when
$\Psi(u) = \frac{1}{\q}\paren*{1
  - \frac{1}{u^{\q}}}, \q \in (0,1)$; and
$\Gamma(t) = n t$ when
$\Psi(u) = 1 -
\frac{1}{u}$. Moreover, $\sur_{\mathrm{comp}}$ with these choices of $\Psi$ are Bayes-consistent with respect to $\sfL$.
\end{restatable}
\vskip -0.02in
The proof of Theorem~\ref{Thm:comp-bound} is included in Appendix~\ref{app:comp-bound}. Similar to the proof of Theorem~\ref{Thm:log-bound-general}, we make use of Theorem~\ref{Thm:tool-Gamma} and Lemma~\ref{lemma:delta_target} in Appendix~\ref{app:tool}. However, upper bounding the conditional regret of $\sfL$ by that of the surrogate loss $\sur_{\rm{comp}}$ for different choices of $\Psi$ requires a distinct analysis depending on the specific form of the function $\Psi$, leading to various concave functions $\Gamma$. Our proof is inspired by the proof of $\sH$-consistency bounds for comp-sum losses in \citep{mao2023cross} through the introduction of a parameter $\mu$ and optimization. However, the novelty lies in the adaptation of $\mu$ with a quantity $\s$ tailored to multi-label loss functions instead of the score vector $h$ itself.

Note that, as with $\Psi(u) = \log(u)$ shown in Section~\ref{sec:general}, for $\Psi(u) = u - 1$, the bounds are also independent of the number of labels and are favorable. However, for other choices of $\Psi$, the bounds exhibit a worse dependency on $n$, which can be exponential with respect to $\num$.

In Appendix~\ref{sec:cstnd}, we introduce another novel family of surrogate losses, adapting \emph{constrained losses} \citep{lee2004multicategory}  from standard classification
to \emph{multi-label constrained losses} in a similar way.

\section{Efficient Gradient Computation}
\label{sec:algo}
\vskip -0.01in

In this section, we demonstrate the efficient computation of the gradient for the multi-label logistic loss $\sur_{\log}$ at any point $(x^j, y^j)$. This loss function is therefore both theoretically grounded in $\sH$-consistency bounds and computationally efficient. Consider the labeled pair $(x^j, y^j)$ and a hypothesis $h$ in $\sH$. The expression for $\sur_{\log}(h, x^j, y^j)$ can be reformulated as follows:
\begin{align*}
\sur_{\log}(h, x^j, y^j)
& = \sum_{\by' \in \sY} \paren*{ 1- \ov \sfL(\by', y^j) } \log \paren*{\sum_{y'' \in \sY} e^{\sum_{i = 1}^\num \paren*{ y''_i - y'_i} h(x^j, i)}}\\
& = - \sum_{y' \in \sY} \paren*{ 1- \ov \sfL(y', y^j) } \sum_{i = 1}^\num y'_i h(x^j, i) + \sum_{y' \in \sY} \paren*{ 1- \ov \sfL(y', y^j) }  \log \paren*{\sum_{y \in \sY} e^{\sum_{i = 1}^\num  y_i h(x^j, i)}}.
\end{align*}
Let $\sfL_1(j) = \sum_{y \in \sY} \paren*{ 1- \ov \sfL(y, y^j) }$, which is independent of $h$ and can be pre-computed. It can also be invariant with respect to $j$ and is a fixed constant for many loss functions such as the Hamming loss.
Next, we will consider the hypothesis set of linear functions $\sH  =  \curl*{x
  \mapsto \bw \cdot \Psi(x, i) \colon \bw \in \Rset^d}$, where $\Psi$
is a feature mapping from $\sX \times [\num]$ to $\Rset^{d}$. Using the shorthand $\bw$ for $h$, we can rewrite $\sur_{\log}$
at $(x^j, y^j)$ as follows:
\begin{align}
\sur_{\log}(\bw, x^j, y^j)
& = - \bw \cdot \bracket*{\sum_{y' \in \sY} \paren*{ 1- \ov \sfL(y', y^j) } \paren*{\sum_{i = 1}^\num y'_i  \Psi(x^j, i)} } + \sfL_1(j) \log \paren*{Z_{\bw, j}},
\end{align}
where $Z_{\bw, j} = \sum_{y \in \sY} e^{\bw \cdot \paren*{\sum_{i = 1}^\num y_i  \Psi(x^j, i)}} $. Then, we can compute the gradient of $\sur_{\log}$ at any $\bw \in \Rset^d$:
\begin{align}
\label{eq:gradient}
\nabla \sur_{\log}(\bw) 
&  =  -\sum_{y' \in \sY} \paren*{ 1- \ov \sfL(y', y^j) } \paren*{\sum_{i = 1}^\num y'_i  \Psi(x^j, i)}  +
 \sfL_1(j)  \sum_{y\in \sY} \frac{e^{\bw \cdot \paren*{\sum_{i = 1}^\num y_i  \Psi(x^j, i)}} }{ Z_{\bw, j} } \paren*{\sum_{i = 1}^\num y_i  \Psi(x^j, i)} \nonumber \\
&  =  -\sum_{y' \in \sY} \paren*{ 1- \ov \sfL(y', y^j) } \paren*{\sum_{i = 1}^\num y'_i  \Psi(x^j, i)}
+  \sfL_1(j)  \E_{y \sim \qq_\bw}\bracket*{ \paren*{\sum_{i = 1}^\num y_i  \Psi(x^j, i)}},
\end{align}
where $\qq_\bw$ is a distribution over $\sY$ with probability mass function $\qq_\bw(y) = \frac{e^{\bw \cdot \paren*{\sum_{i = 1}^\num y_i  \Psi(x^j, i)}} }{ Z_{\bw, j} }$. By rearranging the terms in \eqref{eq:gradient}, we obtain the following result.

\begin{lemma}
\label{lemma:gradient}
The gradient of $\sur_{\log}$ at any $\bw \in \Rset^d$ can be expressed as follows:
\begin{equation*}
\nabla \sur_{\log}(\bw) 
 =  \sum_{i = 1}^\num \Psi(x^j, i) \sfL_2(i, j)
 +
 \sfL_1(j)  \sum_{i = 1}^\num \Psi(x^j, i)  Q_{\bw}(i)\\
\end{equation*}
where $\sfL_2(i, j) = \sum_{y \in \sY} \paren*{ 1- \ov \sfL(y, y^j) } y_i$, $\sfL_1(j) = \sum_{y \in \sY} \paren*{ 1- \ov \sfL(y, y^j) }$, $Q_{\bw}(i) = \sum_{y \in \sY} \qq_{\bw}(y) y_i$, $\qq_\bw(y) = \frac{e^{\bw \cdot \paren*{\sum_{i = 1}^\num y_i  \Psi(x^j, i)}} }{ Z_{\bw, j} }$, and $Z_{\bw, j} = \sum_{y \in \sY} e^{\bw \cdot \paren*{\sum_{i = 1}^\num y_i  \Psi(x^j, i)}} $. The overall time
complexity for gradient computation is $O(\num)$.
\end{lemma}
Here, the evaluation of $\sfL_2(i, j)$, $i \in [\num]$ and $\sfL_1(j)$ can be computed once and for all, before any gradient computation. For evaluation of $Q_{\bw}(i)$, note that it can be equivalently written as follows:
\begin{equation*}
Q_{\bw}(i) = \sum_{y \in \sY} \frac{e^{\bw \cdot \wt \Psi(x^j, y)} }{ \sum_{y \in \sY} e^{\bw \cdot \wt \Psi(x^j, y)} } y_i, \text{ with } \wt \Psi(x^j, y) =  \sum_{i = 1}^\num y_i  \Psi(x^j, i),
\end{equation*}
where $\wt \Psi(x^j, y)$ admits a \emph{Markovian property of order $1$} \citep{manning1999foundations,CortesKuznetsovMohriYang2016}. Thus, as shown by \citet{CortesKuznetsovMohriYang2016,CortesKuznetsovMohriStorcheusYang2018}, $ Q_{\bw}(i)$  can be evaluated efficiently by running two single-source
shortest-distance algorithms over the $(+, \times)$ semiring on an
appropriate weighted finite automaton (WFA). More specifically, in our case, the WFA can be described as follows: there are $(\num + 1)$ vertices labeled $0, \ldots, l$. There are two transitions from $k$ to $(k + 1)$ labeled with $+1$ and $-1$. The weight of the transition with label $+1$ is $\exp(+\bw \cdot \wt \Psi(x^j, k))$, and $\exp(-\bw \cdot \wt \Psi(x^j, k))$ for the other. $0$ is the initial state, and $\num$ the final state.
The overall time
complexity of computing all quantities $Q_{\bw}(i)$, $i \in [\num]$, is $O(\num)$.

\section{Conclusion}
\label{sec:conclusion}

We presented a comprehensive analysis of surrogate losses for
multi-label learning, establishing strong consistency guarantees. We
introduced a novel multi-label logistic loss that addresses the
shortcomings of existing methods and enjoys label-independent
consistency bounds. Our proposed family of multi-label comp-sum losses
offers a unified framework with strong consistency guarantees for any
general multi-label loss, significantly expanding upon previous
work. Additionally, we presented efficient algorithms for their
gradient computation. This unified framework holds promise for broader
applications and opens new avenues for future research in multi-label
learning and related areas.


\bibliography{mll}

\newpage
\appendix

\renewcommand{\contentsname}{Contents of Appendix}
\tableofcontents
\addtocontents{toc}{\protect\setcounter{tocdepth}{4}} 
\clearpage

\section{Extension: multi-label constrained losses}
\label{sec:cstnd}

In this section, we introduce another novel family of surrogate losses, adapting \emph{constrained losses} \citep{lee2004multicategory,awasthi2022multi}  from standard classification
to \emph{multi-label constrained losses} in a similar way. Given any general multi-label loss $\sfL$, we define \emph{multi-label constrained losses} as:
\begin{equation}
\label{eq:cstnd}
\forall (h, x, y) \in \sH \times \sX \times \sY, \quad \sur_{\rm{cstnd}}(h, x, y) = \sum_{y' \in \sY}  \ov \sfL(y', y) \Phi \paren*{- \sum_{i = 1}^\num y'_i h(x, i)}.
\end{equation}
where $\sum_{y \in \sY} \sum_{i = 1}^\num y_i h(x, i) = 0$. Next, we show that $\sur_{\rm{cstnd}}$ also benefit
from $\sH$-consistency bounds and thus Bayes-consistency for any
multi-label loss.
\begin{restatable}{theorem}{BoundCstnd}
\label{Thm:cstnd-bound}
Let $\sH = \sF^\num$. Assume that $\sF$ is complete
Then, the following $\sH$-consistency bound holds in the multi-label learning:
\begin{equation*}
\forall h \in \sH, \quad \sR_{\sfL}(h) - \sR^*_{\sfL}(\sH) + \sM_{\sfL}(\sH) \leq \Gamma \paren*{ \sR_{\sur_{\rm{cstnd}}}(h) - \sR^*_{\sur_{\rm{cstnd}}}(\sH) + \sM_{\sur_{\rm{cstnd}}}(\sH) },    
\end{equation*}
where $\Gamma(t) = 2\sqrt{\sfL_{\rm{max}}t}$ when $\Phi(u) = e^{-u}$; $\Gamma(t) = 2\sqrt{t}$ when $\Phi(u) = \max\curl*{0, 1 - u}^2$; and $\Gamma(t) = t$ when $\Phi(u) = \max\curl*{0,1 - u}$ or $\Phi(u) = \min\curl*{\max\curl*{0, 1 - u/\rho}, 1}$, $\rho > 0$.
\end{restatable}

\begin{restatable}{corollary}{BoundCstndAll}
\label{Thm:cstnd-bound-all}
The following excess error bound holds in the multi-label learning:
\begin{equation*}
\forall h \in \sH_{\rm{all}}, \quad \sR_{\sfL}(h) - \sR^*_{\sfL}( \sH_{\rm{all}})\leq \Gamma \paren*{ \sR_{\sur_{\rm{cstnd}}}(h) - \sR^*_{\sur_{\rm{cstnd}}}( \sH_{\rm{all}})},    
\end{equation*}
where $\Gamma(t) = 2\sqrt{\sfL_{\rm{max}}t}$ when $\Phi(u) = e^{-u}$; $\Gamma(t) = 2\sqrt{t}$ when $\Phi(u) = \max\curl*{0, 1 - u}^2$; and $\Gamma(t) = t$ when $\Phi(u) = \max\curl*{0,1 - u}$ or $\Phi(u) = \min\curl*{\max\curl*{0, 1 - u/\rho}, 1}$, $\rho > 0$. Moreover, $\sur_{\mathrm{cstnd}}$ with these choices of $\Phi$ are Bayes-consistent with respect to $\sfL$.
\end{restatable}
The proof of Theorem~\ref{Thm:cstnd-bound} is included in Appendix~\ref{app:cstnd-bound}. As with the proof of Theorem~\ref{Thm:comp-bound}, we use Theorem~\ref{Thm:tool-Gamma} and Lemma~\ref{lemma:delta_target} from Appendix~\ref{app:tool}, and aim to upper bound the conditional regret of $\sfL$ by that of the surrogate losses $\sur_{\mathrm{comp}}$ using various concave functions $\Gamma$. However, the difference lies in our introduction and optimization of a parameter $\mu$ tailored to a quantity $\z$ that is specific to the form of the multi-label constrained loss.

These results show that in cases where minimizability gaps vanish, reducing the estimation error of $\sur_{\mathrm{cstnd}}$ to $\epsilon$ results in the estimation error of target multi-label loss $\sfL$ being upper bounded by either $\sqrt{\epsilon}$ or $\epsilon$, modulo a constant that is independent of the number of labels.

\newpage
\section{Proof of \texorpdfstring{$\sH$}{H}-consistency bounds for existing surrogate losses (Theorem~\ref{Thm:binary-bound})}
\label{app:binary-bound}

\BoundBi*
\begin{proof}
Let $p(y \mid x) = \mathbb{P}(Y = y \mid X = x)$ be the conditional probability of $Y = y$ given $X = x$.  Given a multi-label surrogate loss $\sur$ and a hypothesis set $\sH$, we denote the conditional error by  $\sC_{\sur}(h, x) = \E_{y \mid x} \bracket*{\sur(h, x, y)}$, the best-in-class conditional error by  $\sC^*_{\sur}\paren*{\sH, x} = \inf_{ h \in \sH} \sC_{\sur}(h, x)$, and the conditional regret by $\Delta \sC_{\sur, \sH}\paren*{h, x} = \sC_{\sur}(h, x) - \sC^*_{\sur}\paren*{\sH, x}$.
We can express the conditional error of the hamming loss and the surrogate loss $\sur_{\rm{br}}$ as follows:
\begin{align*}
\sC_{\sfL_{\rm{ham}}}(h, x) 
&= \sum_{y \in \sY} p(y \mid x) \sum_{i = 1}^\num 1_{y_i \neq h(x, i)}\\
&= \sum_{i = 1}^\num \paren*{\sum_{y\colon y_i = +1} p(y \mid x) 1_{1 \neq \sgn(h(x, i))} + \sum_{y\colon y_i = -1} p(y \mid x) 1_{-1 \neq \sgn(h(x, i))} }\\
\sC_{\sur_{\rm{br}}}(h, x) 
&= \sum_{y \in \sY} p(y \mid x) \sum_{i = 1}^\num \Phi(y_i h(x, i))\\
&= \sum_{i = 1}^\num \paren*{\sum_{y\colon y_i = +1} p(y \mid x) \Phi(h(x, i)) + \sum_{y\colon y_i = -1} p(y \mid x) \Phi(-h(x, i)) }
\end{align*}
Let $q(+1 \mid x) = \sum_{y\colon y_i = +1} p(y \mid x)$ and $q(-1 \mid x) = \sum_{y\colon y_i = +=-1} p(y \mid x)$. Let $f_i = h(\cdot, i) \in \sF$, for all $i \in [\num]$. Then, it is clear that the conditional regrets of $\sfL_{\rm{ham}}$ and $\sur_{\rm{br}}$ can be expressed as the corresponding conditional regrets of $\ell_{0-1}$ and $\Phi$ under this new introduced new distribution:
\begin{equation*}
\Delta \sC_{\sfL_{\rm{ham}}, \sH}(h, x) = \sum_{i = 1}^\num  \Delta \sC_{\ell_{0-1}, \sF}(f_i, x), \quad \Delta \sC_{\sur_{\rm{br}}, \sH}(h, x) = \sum_{i = 1}^\num  \Delta \sC_{\Phi, \sF}(f_i, x).
\end{equation*}
Since we have $\Delta \sC_{\ell_{0-1}, \sF}(f_i, x) \leq \Gamma \paren*{ \Delta \sC_{\Phi, \sF}(f_i, x)}$ under the assumption, we obtain
\begin{align*}
\Delta \sC_{\sfL_{\rm{ham}}, \sH}(h, x) 
 = \sum_{i = 1}^\num  \Delta \sC_{\ell_{0-1}, \sF}(f_i, x)
& \leq \sum_{i = 1}^\num \Gamma \paren*{ \Delta \sC_{\Phi, \sF}(f_i, x)}\\
& \leq \num \Gamma \paren*{\frac{1}{\num} \sum_{i = 1}^\num \Delta \sC_{\Phi, \sF}(f_i, x)} \tag{concavity of $\Gamma$}\\
& = \num \Gamma \paren*{\frac{1}{\num}  \Delta \sC_{\sur_{\rm{br}}, \sH}(h, x)}.
\end{align*}
By taking the expectation on both sides and using the Jensen's inequality, we complete the proof.
\end{proof}

\newpage

\section{Proofs of \texorpdfstring{$\sH$}{H}-consistency bounds for new surrogate losses}
\label{app:new}

\subsection{Auxiliary definitions and results (Theorem~\ref{Thm:tool-Gamma} and Lemma~\ref{lemma:delta_target})}
\label{app:tool}

Before proceeding with the proof, we first introduce some notation and definitions. Given a multi-label surrogate loss $\sur$ and a hypothesis set $\sH$, we denote the conditional error by  $\sC_{\sur}(h, x) = \E_{y \mid x} \bracket*{\sur(h, x, y)}$, the best-in-class conditional error by  $\sC^*_{\sur}\paren*{\sH, x} = \inf_{ h \in \sH} \sC_{\sur}(h, x)$, and the conditional regret by $\Delta \sC_{\sur, \sH}\paren*{h, x} = \sC_{\sur}(h, x) - \sC^*_{\sur}\paren*{\sH, x}$. We then present a general theorem, which shows that to derive $\sH$-consistency bounds in multi-label learning with a concave function $\Gamma$, it is only necessary to upper bound the conditional regret of the target multi-label loss by that of the surrogate loss with the same $\Gamma$.

\begin{restatable}
  {theorem}{Tool-Gamma}
\label{Thm:tool-Gamma}
Let $\sfL$ be a multi-label loss and $\sur$ be a surrogate loss.
Given a concave function $\Gamma \colon \Rset_{+} \to \Rset_{+}$. If the following condition holds for all $h \in \sH$ and $x \in \sX$:
\begin{equation}
\label{eq:c-regret-Gamma}
 \Delta \sC_{\sfL, \sH}(h, x) \leq \Gamma
\paren*{\Delta \sC_{\sur, \sH}(h, x)},
\end{equation}
then, for any distribution and for all hypotheses $h \in \sH$,
    \begin{equation*}
     \sR_{\sfL}(h)- \sR^*_{\sfL}(\sH) + \sM_{\sfL}(\sH) \leq \Gamma \paren*{\sR_{\sur}(h) - \sR^*_{\sur}(\sH) + \sM_{\sur}(\sH)}.
    \end{equation*}
\end{restatable}
\begin{proof}
By the definitions, the expectation of the conditional regrets for $\sfL$ and $\sur$ can be expressed as:
\begin{align*}
\E_{x} \bracket*{\Delta \sC_{\sfL, \sH}(h, x)} & = \sR_{\sfL}(h) - \sR^*_{\sfL}(\sH) + \sM_{\sfL}(\sH)\\
\E_{x} \bracket*{\Delta \sC_{\sur, \sH}(h, x)} & = \sR_{\sur}(h) - \sR^*_{\sur}(\sH) + \sM_{\sur}(\sH).
\end{align*}
Thus, by taking the expectation on both sides of \eqref{eq:c-regret-Gamma} and using Jensen's inequality, we have
\begin{align*}
\sR_{\sfL}(h) - \sR^*_{\sfL}(\sH) + \sM_{\sfL}(\sH) 
& = \E_{x} \bracket*{\Delta \sC_{\sfL, \sH}(h, x)}\\
& \leq \E_{x} \bracket*{ \Gamma
\paren*{\Delta \sC_{\sur, \sH}(h, x)} } \tag{Eq. \eqref{eq:c-regret-Gamma}}\\
& \leq 
\Gamma \paren*{\E_{x} \bracket*{\Delta \sC_{\sur, \sH}(h, x)}} \tag{concavity of $\Gamma$}\\
& = \Gamma \paren*{\sR_{\sur}(h) - \sR^*_{\sur}(\sH) + \sM_{\sur}(\sH)}.
\end{align*}
This completes the proof.
\end{proof}
To derive $\sH$-consistency bounds using Theorem~\ref{Thm:tool-Gamma}, we will characterize the conditional regret of a multi-label loss $\sfL$. For simplicity, we first introduce some notation. For any $x \in \sX$, let $\yy(x) = \argmin_{y' \in \sY} \E_{y \mid x} \bracket*{\ov \sfL(y', y)} \in \sY$. To simplify the notation further, we will drop the dependency on $x$. Specifically, we use $\yy$ to denote $\yy(x)$ and $\hh$ to denote $\hh(x)$.
Additionally, we define $\ch = \E_{y \mid x} \bracket*{\ov \sfL(\hh, y)} $, $\cy = \E_{y \mid x} \bracket*{\ov \sfL(\yy, y)}$ and  $\cyy = \E_{y \mid x} \bracket*{\ov \sfL(y', y)}$, $\forall y' \in \sY$.
\begin{lemma}
\label{lemma:delta_target}
Let $\sH = \sF^{\num}$. Assume that $\sF$ is complete. Then, the conditional regret of a multi-label loss $\sfL$ can be expressed as follows:
$
\Delta \sC_{\sfL, \sH}(h, x) = \ch - \cy.
$
\end{lemma}
\begin{proof}
By definition, the conditional error of $\sfL$ can be expressed as follows: 
\begin{equation*}
\sC_{\sfL}(h, x) = \E_{y \mid x} \bracket*{\sfL(h, x, y)} = \E_{y \mid x} \bracket*{\ov \sfL(\hh(x), y)} = \ch.
\end{equation*}
Since $\sH = \sF^{\num}$ and $\sF$ is complete, for any $x \in \sX$, $\curl*{\hh(x) \colon h \in \sH} = \sY$. Then, the best-in-class conditional error of $\sfL$ can be expressed as follows: 
\begin{equation}
\sC^*_{\sfL}\paren*{\sH, x} = \inf_{ h \in \sH} \sC_{\sfL}(h, x) = \inf_{ h \in \sH} \E_{y \mid x} \bracket*{\ov \sfL(\hh(x), y)} = \E_{y \mid x} \bracket*{\ov \sfL(\yy(x), y)} = \cy.
\end{equation}
Therefore, $\Delta \sC_{\sfL, \sH}(h, x) = \sC_{\sfL}(h, x) - \sC^*_{\sfL}\paren*{\sH, x} = \ch - \cy.$
\end{proof}

Next, by using Lemma~\ref{lemma:delta_target}, we will upper bound the conditional regret of the target multi-label loss $\sfL$ by that of the surrogate loss $\sur$ with a concave function $\Gamma$.

\subsection{Proof of Theorem~\ref{Thm:log-bound}}
\label{app:log-bound}

\BoundLog*

\begin{proof}
We will use the following notation adapted to the Hamming loss: $\ch = \E_{y \mid x} \bracket*{\ov \sfL_{\rm{ham}}(\hh, y)} $, $\cy = \E_{y \mid x} \bracket*{\ov \sfL(\yy, y)}$ and  $\cyy = \E_{y \mid x} \bracket*{\ov \sfL_{\rm{ham}}(y', y)}$, $\forall y' \in \sY$. We will denote by $\s(h, x, y') = \frac{e^{\sum_{i = 1}^\num y'_i h(x, i)}}{ \sum_{y'' \in \sY} e^{\sum_{i = 1}^\num  y''_i h(x, i)} }$  and simplify notation by using $\syy$, thereby dropping the dependency on $h$ and $x$. It is clear that $\syy \in [0, 1]$. Then, the conditional error of $\sur_{\log}$ can be expressed as follows:
\begin{align*}
\sC_{\sur_{\log}}(h, x) 
& = \E_{y \mid x} \bracket*{\sum_{y' \in \sY} \paren*{ 1- \ov \sfL(y', y) } \log \paren*{\sum_{y'' \in \sY} e^{\sum_{i = 1}^\num \paren*{ y''_i - y'_i} h(x, i)}}}\\
& = -\sum_{y' \in \sY} (1 - \cyy) \log(\syy)
\end{align*}
For any $\hh \neq \yy$, we define $\s^{\mu}$ as follows: set $\s^{\mu}_{y'} = \syy$ for all $y' \neq \yy$ and $y' \neq \hh$; define $\s^{\mu}_{\hh} = \sy - \mu$; and let $\s^{\mu}_{\yy} = \sh + \mu$. Note that $\s^{\mu}$ can be realized by some $h' \in \sH$ under the assumption.  Then, we have
\begin{align*}
\Delta \sC_{\sur_{\log}, \sH}(h, x)
& \geq \paren*{-\sum_{y' \in \sY} (1 - \cyy) \log(\syy)} - \inf_{\mu \in \Rset} \paren*{-\sum_{y' \in \sY} (1 - \cyy) \log\paren*{\s^{\mu}_{y'}}}\\
& = \sup_{\mu \in \Rset} \curl*{ (1 - \ch) \bracket*{ \log(\sy - \mu) - \log(\sh)} + (1 - \cy) \bracket*{ \log(\sh + \mu) - \log(\sy) } }\\
& = (1 - \cy) \log \frac{\paren*{\sh + \sy}(1 - \cy)}{\sy \paren*{2 - \ch - \cy}} + (1 - \ch) \log \frac{\paren*{\sh + \sy}(1 - \ch)}{\sh \paren*{2 - \ch - \cy}}\\
\tag{supremum is attained when $\mu^*  = 
\frac{-(1 - \ch) \sh + (1 - \cy) \sy}{2 - \cy - \ch}$}\\
& \geq (1 - \cy) \log \frac{2(1 - \cy)}{\paren*{2 - \ch - \cy}} + (1 - \ch) \log \frac{2(1 - \ch)}{\paren*{2 - \ch - \cy}}\\
\tag{minimum is attained when $\sh  = \sy$}\\
& \geq \frac{(\ch - \cy)^2}{2 \paren*{2 - \ch - \cy}}
\tag{$a\log \frac{2a}{a + b} + b\log \frac{2b}{a + b}\geq \frac{(a - b)^2}{2(a + b)}, \forall a, b \in[0, 1]$}\\
& \geq \frac{(\ch - \cy)^2}{4}.
\end{align*}
Therefore, by Lemma~\ref{lemma:delta_target}, $\Delta \sC_{\sfL_{\rm{ham}}, \sH}(h, x) \leq 2 \paren*{\Delta \sC_{\sur_{\log}, \sH}(h, x)}^{\frac12}$. By Theorem~\ref{Thm:tool-Gamma}, we complete the proof.
\end{proof}

\newpage
\subsection{Proof of Theorem~\ref{Thm:log-bound-general}}
\label{app:log-bound-general}

\BoundLog*

\begin{proof}
The proof is basically the same as that of Theorem~\ref{Thm:log-bound}, modulo replacing the Hamming loss $\sfL_{\rm{ham}}$ with a general multi-label loss $\sfL$. We adopt the following notation: $\ch = \E_{y \mid x} \bracket*{\ov \sfL(\hh, y)} $, $\cy = \E_{y \mid x} \bracket*{\ov \sfL(\yy, y)}$ and  $\cyy = \E_{y \mid x} \bracket*{\ov \sfL(y', y)}$, $\forall y' \in \sY$. We will denote by $\s(h, x, y') = \frac{e^{\sum_{i = 1}^\num y'_i h(x, i)}}{ \sum_{y'' \in \sY} e^{\sum_{i = 1}^\num  y''_i h(x, i)} }$  and simplify notation by using $\syy$, thereby dropping the dependency on $h$ and $x$. It is clear that $\syy \in [0, 1]$. Then, the conditional error of $\sur_{\log}$ can be expressed as follows:
\begin{align*}
\sC_{\sur_{\log}}(h, x) 
& = \E_{y \mid x} \bracket*{\sum_{y' \in \sY} \paren*{ 1- \ov \sfL(y', y) } \log \paren*{\sum_{y'' \in \sY} e^{\sum_{i = 1}^\num \paren*{ y''_i - y'_i} h(x, i)}}}\\
& = -\sum_{y' \in \sY} (1 - \cyy) \log(\syy)
\end{align*}
For any $\hh \neq \yy$, we define $\s^{\mu}$ as follows: set $\s^{\mu}_{y'} = \syy$ for all $y' \neq \yy$ and $y' \neq \hh$; define $\s^{\mu}_{\hh} = \sy - \mu$; and let $\s^{\mu}_{\yy} = \sh + \mu$. Note that $\s^{\mu}$ can be realized by some $h' \in \sH$ under the assumption.  Then, we have
\begin{align*}
\Delta \sC_{\sur_{\log}, \sH}(h, x)
& \geq \paren*{-\sum_{y' \in \sY} (1 - \cyy) \log(\syy)} - \inf_{\mu \in \Rset} \paren*{-\sum_{y' \in \sY} (1 - \cyy) \log\paren*{\s^{\mu}_{y'}}}\\
& = \sup_{\mu \in \Rset} \curl*{ (1 - \ch) \bracket*{ \log(\sy - \mu) - \log(\sh)} + (1 - \cy) \bracket*{ \log(\sh + \mu) - \log(\sy) } }\\
& = (1 - \cy) \log \frac{\paren*{\sh + \sy}(1 - \cy)}{\sy \paren*{2 - \ch - \cy}} + (1 - \ch) \log \frac{\paren*{\sh + \sy}(1 - \ch)}{\sh \paren*{2 - \ch - \cy}}
\tag{supremum is attained when $\mu^*  = 
\frac{-(1 - \ch) \sh + (1 - \cy) \sy}{2 - \cy - \ch}$}\\
& \geq (1 - \cy) \log \frac{2(1 - \cy)}{\paren*{2 - \ch - \cy}} + (1 - \ch) \log \frac{2(1 - \ch)}{\paren*{2 - \ch - \cy}}\\
\tag{minimum is attained when $\sh  = \sy$}\\
& \geq \frac{(\ch - \cy)^2}{2 \paren*{2 - \ch - \cy}}
\tag{$a\log \frac{2a}{a + b} + b\log \frac{2b}{a + b}\geq \frac{(a - b)^2}{2(a + b)}, \forall a, b \in[0, 1]$}\\
& \geq \frac{(\ch - \cy)^2}{4}.
\end{align*}
Therefore, by Lemma~\ref{lemma:delta_target}, $\Delta \sC_{\sfL, \sH}(h, x) \leq 2 \paren*{\Delta \sC_{\sur_{\log}, \sH}(h, x)}^{\frac12}$. By Theorem~\ref{Thm:tool-Gamma}, we complete the proof.
\end{proof}

\newpage
\subsection{Proof of Theorem~\ref{Thm:comp-bound}}
\label{app:comp-bound}

\BoundComp*
\begin{proof}
Recall that we adopt the following notation: $\ch = \E_{y \mid x} \bracket*{\ov \sfL(\hh, y)} $, $\cy = \E_{y \mid x} \bracket*{\ov \sfL(\yy, y)}$ and  $\cyy = \E_{y \mid x} \bracket*{\ov \sfL(y', y)}$, $\forall y' \in \sY$. We will denote by $\s(h, x, y') = \frac{e^{\sum_{i = 1}^\num y'_i h(x, i)}}{ \sum_{y'' \in \sY} e^{\sum_{i = 1}^\num  y''_i h(x, i)} }$  and simplify notation by using $\syy$, thereby dropping the dependency on $h$ and $x$. It is clear that $\syy \in [0, 1]$. Next, we will analyze case by case.

\textbf{The case where $\Phi(u) = \log(u)$}: See the proof of Theorem~\ref{Thm:log-bound-general}.

\textbf{The case where $\Phi(u) = u - 1$}: The conditional error of $\sur_{\rm{comp}}$ can be expressed as follows:
\begin{align*}
& \sC_{\sur_{\rm{comp}}}(h, x) \\
& = \E_{y \mid x} \bracket*{\sum_{y' \in \sY} \paren*{ 1- \ov \sfL(y', y) } \paren*{\sum_{y'' \in \sY} e^{\sum_{i = 1}^\num \paren*{ y''_i - y'_i} h(x, i)} - 1 }}\\
& = \sum_{y' \in \sY} (1 - \cyy) \paren*{\frac{1}{\syy} - 1}.
\end{align*}
For any $\hh \neq \yy$, we define $\s^{\mu}$ as follows: set $\s^{\mu}_{y'} = \syy$ for all $y' \neq \yy$ and $y' \neq \hh$; define $\s^{\mu}_{\hh} = \sy - \mu$; and let $\s^{\mu}_{\yy} = \sh + \mu$. Note that $\s^{\mu}$ can be realized by some $h' \in \sH$ under the assumption.  Then, we have
\begin{align*}
& \Delta \sC_{\sur_{\rm{comp}}, \sH}(h, x)\\
& \geq \sum_{y' \in \sY} (1 - \cyy) \paren*{\frac{1}{\syy} - 1} - \inf_{\mu \in \Rset} \paren*{\sum_{y' \in \sY} (1 - \cyy) \paren*{\frac{1}{\s^{\mu}_{y'}} - 1} }\\
& = \sup_{\mu \in \Rset} \curl*{ (1 - \ch) \bracket*{\frac{1}{\sh} - \frac{1}{\sy - \mu} } + (1 - \cy) \bracket*{ \frac{1}{\sy} - \frac{1}{\sh + \mu} } }\\
& = \frac{1 - \ch}{\sh} + \frac{1 - \cy}{\sy} - \frac{2 - \ch - \cy + 2(1 - \ch)^\frac{1}{2}(1 - \cy)^\frac{1}{2}}{\sh + \sy} 
\tag{supremum is attained when $\mu^*  =   \frac{-\sqrt{1- \ch} \sh + \sqrt{1 - \cy}\sy}{\sqrt{1 - \cy} + \sqrt{1 - \ch}}$}\\
& \geq \paren*{(1 - \ch)^\frac{1}{2} - (1 - \cy)^\frac{1}{2}}^2
\tag{minimum is attained when $\sh  = \sy  =  \frac{1}{2}$}\\
& = \frac{(\ch - \cy)^2}{\paren*{(1 - \ch)^\frac{1}{2} + (1 - \cy)^\frac{1}{2}}^2}\\
& \geq \frac{(\ch - \cy)^2}{4}.
\end{align*}
Therefore, by Lemma~\ref{lemma:delta_target}, $\Delta \sC_{\sfL, \sH}(h, x) \leq 2 \paren*{\Delta \sC_{\sur_{\rm{comp}}, \sH}(h, x)}^{\frac12}$. By Theorem~\ref{Thm:tool-Gamma}, we complete the proof.

\textbf{The case where $\Phi(u) = \frac{1}{q}\paren*{1 - \frac{1}{u^q}}, q \in (0, 1)$}: The conditional error of $\sur_{\rm{comp}}$ can be expressed as:
\begin{align*}
\sC_{\sur_{\rm{comp}}}(h, x) = \frac{1}{q} \sum_{y' \in \sY} (1 - \cyy) \paren*{1 - (\syy)^q}.
\end{align*}
For any $\hh \neq \yy$, we define $\s^{\mu}$ as follows: set $\s^{\mu}_{y'} = \syy$ for all $y' \neq \yy$ and $y' \neq \hh$; define $\s^{\mu}_{\hh} = \sy - \mu$; and let $\s^{\mu}_{\yy} = \sh + \mu$. Note that $\s^{\mu}$ can be realized by some $h' \in \sH$ under the assumption.  Then, we have
\begin{align*}
& \Delta \sC_{\sur_{\rm{comp}}, \sH}(h, x)\\
& \geq \frac{1}{q} \sum_{y' \in \sY} (1 - \cyy) \paren*{1 - \sy} - \inf_{\mu \in \Rset} \paren*{\frac{1}{q} \sum_{y' \in \sY} (1 - \cyy) \paren*{1 - (\s^{\mu}_{y'})^q} }\\
& = \frac{1}{q} \sup_{\mu \in \Rset} \curl*{ (1 - \ch) \bracket*{-\sh + (\sy - \mu)^{q} } + (1 - \cy) \bracket*{ -(\sy)^{q} + (\sh + \mu)^{q} } }\\
&   =  \frac{1}{q}\paren*{\sh+ \sy}^{q}\paren*{(1 - \cy)^{\frac{1}{1 - q}} + (1 - \ch)^{\frac{1}{1 - q}}}^{1 - q} - \frac{1}{q}(1 - \cy)\sy^{q} - \frac{1}{q}(1 - \ch)\sh^{q}
\tag{supremum is attained when $\mu^*  =   \frac{-(1 - \ch)^{\frac{1}{1 - q}}\sh + (1 - \cy)^{\frac{1}{1 - q}}\sy}{(1 - \cy)^{\frac{1}{1 - q}}+(1 - \ch)^{\frac{1}{1 - q}}}$}
\\
& \geq \frac{1}{q n^{q}} \bracket*{2^{q} \paren*{(1 - \cy)^{\frac{1}{1 - q}} + (1 - \ch)^{\frac{1}{1 - q}}}^{1 - q} - (1 - \cy)-(1 - \ch)}
\tag{minimum is attained when $\sh  = \sy  =  \frac{1}{n}$}\\
& \geq \frac{(\ch - \cy)^2}{4n^{q}}
\tag{$\paren*{\frac{a^{\frac{1}{1 - q}} + b^{\frac{1}{1 - q}}}{2}}^{1 - q} - \frac{a + b}{2}  \geq \frac{q}{4}(a - b)^2, \forall a, b\in[0, 1]$, $0 \leq a + b \leq 1$}.
\end{align*}
Therefore, by Lemma~\ref{lemma:delta_target}, $\Delta \sC_{\sfL, \sH}(h, x) \leq 2 n^{\frac{q}{2}}\paren*{\Delta \sC_{\sur_{\rm{comp}}, \sH}(h, x)}^{\frac12}$. By Theorem~\ref{Thm:tool-Gamma}, we complete the proof.

\textbf{The case where $\Phi(u) = \paren*{1 - \frac{1}{u}}$}: The conditional error of $\sur_{\rm{comp}}$ can be expressed as:
\begin{align*}
\sC_{\sur_{\rm{comp}}}(h, x) =  \sum_{y' \in \sY} (1 - \cyy) \paren*{1 - (\syy)^q}.
\end{align*}
For any $\hh \neq \yy$, we define $\s^{\mu}$ as follows: set $\s^{\mu}_{y'} = \syy$ for all $y' \neq \yy$ and $y' \neq \hh$; define $\s^{\mu}_{\hh} = \sy - \mu$; and let $\s^{\mu}_{\yy} = \sh + \mu$. Note that $\s^{\mu}$ can be realized by some $h' \in \sH$ under the assumption.  Then, we have
\begin{align*}
& \Delta \sC_{\sur_{\rm{comp}}, \sH}(h, x)\\
& \geq  \sum_{y' \in \sY} (1 - \cyy) \paren*{1 - \sy} - \inf_{\mu \in \Rset} \paren*{ \sum_{y' \in \sY} (1 - \cyy) \paren*{1 - \s^{\mu}_{y'}} }\\
& =  \sup_{\mu \in \Rset} \curl*{ (1 - \ch) \bracket*{-\sh + \sy - \mu } + (1 - \cy) \bracket*{ -\sy + \sh + \mu } }\\
&   = \sh (\ch - \cy)
\tag{supremum is attained when $\mu^*  = \sy$}
\\
& \geq \frac1n (\ch - \cy)
\tag{minimum is attained when $\sh = \frac{1}{n}$}.
\end{align*}
Therefore, by Lemma~\ref{lemma:delta_target}, $\Delta \sC_{\sfL, \sH}(h, x) \leq n \Delta \sC_{\sur_{\rm{comp}}, \sH}(h, x)$. By Theorem~\ref{Thm:tool-Gamma}, we complete the proof.
\end{proof}


\subsection{Proof of Theorem~\ref{Thm:cstnd-bound}}
\label{app:cstnd-bound}

\BoundCstnd*
\begin{proof}
Recall that we adopt the following notation: $\ch = \E_{y \mid x} \bracket*{\ov \sfL(\hh, y)} $, $\cy = \E_{y \mid x} \bracket*{\ov \sfL(\yy, y)}$ and  $\cyy = \E_{y \mid x} \bracket*{\ov \sfL(y', y)}$, $\forall y' \in \sY$.  We will also denote by $\z(h, x, y') = \sum_{i = 1}^\num  y'_i h(x, i)$  and simplify notation by using $\zyy$, thereby dropping the dependency on $h$ and $x$. It is clear that the constraint can be expressed as $\sum_{y' \in \sY} \zyy = 0$. Next, we will analyze case by case.

\textbf{The case where $\Phi(u) = e^{-u}$}:  The conditional error of $\sur_{\rm{cstnd}}$ can be expressed as follows:
\begin{align*}
\sC_{\sur_{\rm{cstnd}}}(h, x)
= \E_{y \mid x} \bracket*{\sum_{y' \in \sY}  \ov \sfL(y', y) e^{\sum_{i = 1}^\num y'_i h(x, i)}}
= \sum_{y' \in \sY} \cyy e^{\zyy}.
\end{align*}
For any $\hh \neq \yy$, we define $\z^{\mu}$ as follows: set $\z^{\mu}_{y'} = \zyy$ for all $y' \neq \yy$ and $y' \neq \hh$; define $\z^{\mu}_{\hh} = \zy - \mu$; and let $\z^{\mu}_{\yy} = \zh + \mu$. Note that $\z^{\mu}$ can be realized by some $h' \in \sH$ under the assumption.  Then, we have
\begin{align*}
\Delta \sC_{\sur_{\rm{comp}}, \sH}(h, x)
& \geq \sum_{y' \in \sY} \cyy e^{\zyy} - \inf_{\mu \in \Rset} \paren*{\sum_{y' \in \sY} \cyy e^{\z^{\mu}_{y'}} }\\
&   =   \sup_{\mu \in \Rset} \curl*{\cy \paren*{e^{\zy} - e^{\zh + \mu}} + \ch \paren*{e^{\zh} - e^{\zy - \mu}}}\\
&   =  \paren*{\sqrt{\ch e^{\zh}} - \sqrt{\cy e^{\zy}}}^2
\tag{supremum is attained when $\mu^*  =  \frac12 \log\frac{\cy e^{\zy}}{\ch e^{\zh}}$}
\\
& = \paren*{\frac{\ch - \cy}{\sqrt{\cy} + \sqrt{\ch}}}^2
\tag{minimum is attained when $\zh  = \zy  =  0$}\\
& \geq \frac1{4 \sfL_{\rm{max}}} \paren*{\ch-\cy}^2.
\end{align*}
Therefore, by Lemma~\ref{lemma:delta_target}, $\Delta \sC_{\sfL, \sH}(h, x) \leq 2 \paren*{\sfL_{\rm{max}}}^{\frac12} \paren*{\Delta \sC_{\sur_{\rm{cstnd}}, \sH}(h, x)}^{\frac12}$. By Theorem~\ref{Thm:tool-Gamma}, we complete the proof.

\textbf{The case where $\Phi(u) = \max\curl*{0, 1 - u}^2$}: The conditional error of $\sur_{\rm{cstnd}}$ can be expressed as follows:
\begin{align*}
& \sC_{\sur_{\rm{cstnd}}}(h, x) = \sum_{y' \in \sY} \cyy \max\curl*{0, 1 + \zyy}^2.
\end{align*}
For any $\hh \neq \yy$, we define $\z^{\mu}$ as follows: set $\z^{\mu}_{y'} = \zyy$ for all $y' \neq \yy$ and $y' \neq \hh$; define $\z^{\mu}_{\hh} = \zy - \mu$; and let $\z^{\mu}_{\yy} = \zh + \mu$. Note that $\z^{\mu}$ can be realized by some $h' \in \sH$ under the assumption.  Then, we have
\begin{align*}
& \Delta \sC_{\sur_{\rm{cstnd}}, \sH}(h, x)\\
& \geq \sum_{y' \in \sY} \cyy \max\curl*{0, 1 + \zyy}^2 - \inf_{\mu \in \Rset} \paren*{\sum_{y' \in \sY} \cyy \max\curl*{0, 1 + \z^{\mu}_{y'}}^2 }\\
&   =   \sup_{\mu\in \Rset} \bigg\{\cy \paren*{\max \curl*{0, 1 + \zy}^2 - \max\curl*{0, 1 + \zh + \mu}^2 } + \ch \paren*{\max\curl*{0, 1 + \zh}^2 - \max\curl*{0, 1 + \zy - \mu}^2}\bigg\}\\
& \geq \paren*{1+ \zh}^2 \paren*{\cy - \ch}^2
\tag{differentiating with respect to $\mu$ to optimize}
\\
& \geq  \paren*{\ch-\cy}^2
\tag{minimum is attained when $\zh = 0$}.
\end{align*}
Therefore, by Lemma~\ref{lemma:delta_target}, $\Delta \sC_{\sfL, \sH}(h, x) \leq \paren*{\Delta \sC_{\sur_{\rm{cstnd}}, \sH}(h, x)}^{\frac12}$. By Theorem~\ref{Thm:tool-Gamma}, we complete the proof.

\textbf{The case where $\Phi(u) = \max\curl*{0, 1 - u}$}: The conditional error of $\sur_{\rm{cstnd}}$ can be expressed as:
\begin{align*}
\sC_{\sur_{\rm{cstnd}}}(h, x) = \sum_{y' \in \sY} \cyy \max\curl*{0, 1 + \zyy}.
\end{align*}
For any $\hh \neq \yy$, we define $\z^{\mu}$ as follows: set $\z^{\mu}_{y'} = \zyy$ for all $y' \neq \yy$ and $y' \neq \hh$; define $\z^{\mu}_{\hh} = \zy - \mu$; and let $\z^{\mu}_{\yy} = \zh + \mu$. Note that $\z^{\mu}$ can be realized by some $h' \in \sH$ under the assumption.  Then, we have
\begin{align*}
& \Delta \sC_{\sur_{\rm{cstnd}}, \sH}(h, x)\\
& \geq \sum_{y' \in \sY} \cyy \max\curl*{0, 1 + \zyy} - \inf_{\mu \in \Rset} \paren*{\sum_{y' \in \sY} \cyy \max\curl*{0, 1 + \z^{\mu}_{y'}} }\\
&   =   \sup_{\mu\in \Rset} \bigg\{\cy \paren*{\max \curl*{0, 1 + \zy} - \max\curl*{0, 1 + \zh + \mu} } + \ch \paren*{\max\curl*{0, 1 + \zh}^2 - \max\curl*{0, 1 + \zy - \mu}^2}\bigg\}\\
& \geq \paren*{1+ \zh} \paren*{\cy - \ch}
\tag{differentiating with respect to $\mu$ to optimize}
\\
& \geq  \paren*{\ch-\cy}
\tag{minimum is attained when $\zh = 0$}.
\end{align*}
Therefore, by Lemma~\ref{lemma:delta_target}, $\Delta \sC_{\sfL, \sH}(h, x) \leq \Delta \sC_{\sur_{\rm{cstnd}}, \sH}(h, x)$. By Theorem~\ref{Thm:tool-Gamma}, we complete the proof.

\textbf{The case where $\Phi(u) = \min\curl*{\max\curl*{0, 1 - u/\rho}, 1}, \rho > 0$}: The conditional error of $\sur_{\rm{cstnd}}$ can be expressed as:
\begin{align*}
\sC_{\sur_{\rm{cstnd}}}(h, x) = \sum_{y' \in \sY} \cyy \min \curl*{\max\curl*{0, 1 + \zyy / \rho}, 1}.
\end{align*}
For any $\hh \neq \yy$, we define $\z^{\mu}$ as follows: set $\z^{\mu}_{y'} = \zyy$ for all $y' \neq \yy$ and $y' \neq \hh$; define $\z^{\mu}_{\hh} = \zy - \mu$; and let $\z^{\mu}_{\yy} = \zh + \mu$. Note that $\z^{\mu}$ can be realized by some $h' \in \sH$ under the assumption.  Then, we have
\begin{align*}
& \Delta \sC_{\sur_{\rm{cstnd}}, \sH}(h, x)\\
& \geq \sum_{y' \in \sY} \cyy \min \curl*{\max\curl*{0, 1 + \zyy / \rho}, 1} - \inf_{\mu \in \Rset} \paren*{\sum_{y' \in \sY} \cyy \min \curl*{\max\curl*{0, 1 + \z_{y'}^{\mu} / \rho}, 1} }\\
&   =   \sup_{\mu\in \Rset} \bigg\{\cy \paren*{\min \curl*{\max \curl*{0, 1 + \zy / \rho}, 1} - \min \curl*{\max\curl*{0, 1 + (\zh + \mu) / \rho}, 1} }
\\
& + \ch \paren*{\min \curl *{\max\curl*{0, 1 + \zh / \rho}, 1} - \min \curl*{\max\curl*{0, 1 + (\zy - \mu) / \rho}}, 1}\bigg\}\\
& \geq \paren*{\cy - \ch}
\tag{differentiating with respect to $\mu$ to optimize}
.
\end{align*}
Therefore, by Lemma~\ref{lemma:delta_target}, $\Delta \sC_{\sfL, \sH}(h, x) \leq \Delta \sC_{\sur_{\rm{cstnd}}, \sH}(h, x)$. By Theorem~\ref{Thm:tool-Gamma}, we complete the proof.
\end{proof}

\section{Future work}
\label{app:future_work}

While our work proposed a unified surrogate loss framework that is Bayes-consistent for any multi-label loss, significantly expanding upon previous work which only established consistency for specific loss functions, empirical comparison with surrogate losses for specific loss functions could be an interesting direction, which we have left for future work and have already initiated. Moreover, the potential to theoretically improve surrogate losses for specific target losses is another promising direction.

\end{document}